\documentclass[12pt,titlepage]{article}
\usepackage{amsmath,amsthm,amssymb}
\usepackage[dvips]{graphicx}
\usepackage{array}
\usepackage[mathscr]{eucal}
\usepackage{eurosym}
\usepackage{subfigure}
\usepackage{color}
\usepackage{listings}
\usepackage{multirow}

\newtheorem{Theorem}{Theorem}[section]

\newtheorem{Proposition}[Theorem]{Proposition}
\newtheorem{Property}[Theorem]{Property}

\theoremstyle{definition}

\newtheorem{Example}[Theorem]{Example}
\theoremstyle{remark}

\numberwithin{equation}{section}
\numberwithin{figure}{section}
\numberwithin{table}{section}

\newcommand{\Cov}{\mathrm{Cov}}

\newcommand{\Var}{\mathrm{Var}}
\newcommand{\Esp}{\mathrm{E}}
\newcommand{\Prob}{\mathrm{P}}
\newcommand{\Ind}{\mathrm{I}}

\newcommand{\bi}{ \begin{itemize}  }
\newcommand{\ei}{\end{itemize}}

\newcommand{\bfor}{ \begin{eqnarray*} }
\newcommand{\efor}{\end{eqnarray*}}

\newcommand{\Xvec}{\boldsymbol{X}}

\newcommand{\xvec}{\boldsymbol{x}}

\begin{document}

\title{AUTOCALIBRATION AND TWEEDIE-DOMINANCE FOR INSURANCE PRICING WITH MACHINE LEARNING}
\author{
Michel Denuit\\
Institute of Statistics, Biostatistics and Actuarial Science\\
UCLouvain\\
Louvain-la-Neuve, Belgium
\\[2mm]
Arthur Charpentier\\
Universit\'e du Qu\'ebec à Montr\'eal (UQAM)\\
Montreal, Quebec, Canada
\\[2mm]
Julien Trufin\\
Department of Mathematics\\
Universit\'e Libre de Bruxelles (ULB)\\
Brussels, Belgium
}

\maketitle

\begin{abstract}
Boosting techniques and neural networks are particularly effective machine learning methods
for insurance pricing.
Often in practice, the sum of fitted values can depart from the observed totals to a large extent.
The possible lack of balance when models are trained by minimizing deviance
outside the familiar GLM with canonical 
link setting has been documented in W\"uthrich (2019, 2020, 2021).
The present paper aims to further study this phenomenon when learning proceeds
by minimizing Tweedie deviance.
It is shown that minimizing deviance involves a trade-off between the
integral of weighted differences of lower partial moments and the bias measured on
a specific scale.
Hence, there is no guarantee that the sum of fitted values stays close to observed totals if the latter
bias term is dominated by the former one entering deviance.
Autocalibration is then proposed as a remedy.
This new method to correct for bias adds an extra local GLM step to the analysis
with the output of the first step as only predictor.
Theoretically, it is shown that it implements the autocalibration concept in pure
premium calculation and ensures that balance also holds on a local scale,
not only at portfolio level as with existing bias-correction techniques.
	\\[3mm]
\end{abstract}

\section{Introduction and motivation}

In the 1960s, North-American actuaries pioneered risk classification with the help of minimum bias methods,
after Bailey and Simon (1960) and Bailey (1963). The central idea is that 
an acceptable set of premiums should reproduce the experience within sub-portfolios corresponding to each 
level of meaningful risk factors (like gender or age, for instance) and also the overall experience, i.e. 
be balanced for each level and in total (leading to the method of marginal totals, or MMT in short).

In the late 1980s, it turned out that for any GLM with canonical link function and a score containing an intercept, 
there is an exact balance between fitted and observed aggregated responses over the whole data set and 
for any level of the categorical features. This formally related GLMs to minimum bias and MMT,
as documented in Mildenhall (1999). This connection greatly facilitated
the wide acceptance of GLMs in actuarial practice that has been particularly fast after the development of
powerful computer tools.
The objective function used for model training (often call loss function in machine 
learning) generally corresponds to deviance or log-likelihood,
in a model accounting for the nature of insurance data under consideration:  typically, Poisson for claim counts,
Gamma for average claim severities and compound Poisson sums with Gamma-distributed terms for claim totals,
all belonging to the Tweedie family with power variance function. 

Actuarial risk classification remained bridged to statistical regression models and naturally 
followed their evolution from GLMs to GAMs, trees and random forests, gradient and statistical boosting, 
projection pursuit and neural networks, to name just a few. This evolution took place gradually, following 
the availability of computational resources in statistical software. The deviance established 
itself as the only objective function, even when the interpretability of the likelihood
equations in terms of balance was lost, and the
underlying MMT caution to GLMs has been progressively forgotten.

However, it turned out that tree-based boosting models and Neural Networks trained to minimize 
deviance often violate total balance, even on the training data set.
This has been documented by W\"uthrich (2019, 2020, 2021).
Advanced learning models are indeed able to produce scores that better correlate
with the response, as well as with the true premium compared to classical GLMs. 
This comes from the additional freedom obtained by letting scores to
depend in a flexible way of available features, not only linearly. But breaking the overall balance is
the price to pay for this higher correlation. 
Because no constraint on the replication of the observed total, or global balance is imposed, 
machine learning tools are also able to substantially increase overall bias.
As pointed out by W\"uthrich (2021), machine learning tools may provide accurate fit
at individual policy level but the global price level may be completely wrong.
It is therefore crucial to correct candidate premiums produced by flexible machine
learning tools before they can be used in practice.

A natural remedy consists in restoring global balance at each step of the iterative procedure
used to optimize the loss function. This is easily implemented
by revising the intercept (under canonical link function) or
by constraining the optimization so that the observed total matches its fitted counterpart.
This simple solution restores global balance but does not ensure that financial equilibrium also
holds in meaningful sub-portfolios (remember that GLMs with canonical link not only imposes global balance,
at portfolio level when the score comprises an intercept, but also at the level of risk classes
determined by binary features).

For this reason, we propose a new strategy based on the concept of autocalibration 
(see, e.g., Kruger and Ziegel, 2020). This approach guarantees global balance as well as local
equilibrium in the spirit of the original MMT. This
simple and effective solution to the problem is implemented by adding an extra step implementing
MMT within a local GLM analysis. Specifically, after the analysis has been performed with a method
that does not necessarily respect marginal totals, a local constant GLM fit is achieved in order
to restore the connection with MMT. 
Thus, as advocated by W\"uthrich (2019, 2020, 2021), we also combine GLMs with advanced 
statistical learning tools but in a different way. 
W\"uthrich (2019, 2020, 2021) used Neural Networks to produce new features in the last hidden layer,
to be used in a GLM replacing the output layer. In this approach, Neural Networks allow actuaries to perform
feature-engineering to feed the GLM score and marginal totals are respected by the use of GLMs in
the last step. W\"uthrich (2019, 2020) then explains how to interpret the working features generated by the last
hidden layer of Neural Networks. This approach can also be related
to the polishing procedure proposed by Zumel (2019) where random forests 
predictions are used to train a linear model
in a second step. As pointed out by this author, the extra step should not 
be performed on the same data in order to avoid a potential source of overfitting (called nested-model bias).
The local GLM approach proposed in this paper applies to any statistical learning
model, not specifically to Neural Networks or random forests. It consists
in using the candidate premium produced in the first step as unique feature in the second autocalibration step. 
This ensures that the feature space reduces to the real line. 
By using a local constant, or intercept-only GLM, the output of the first step defines optimal neighborhoods
to perform local averaging of observed losses.


The remainder of the paper is structured as follows. Section 2 provides the reader with formal
definitions and notation used throughout the text. 
Section 3 recalls Tweedie model and associated deviance. In Section 4, a mixture representation of
Tweedie deviance is derived. Precisely, Tweedie deviance is decomposed into the sum of the
integral of weighted differences of lower partial moments and the bias measured on
a specific scale. This decomposition
is used to understand the consequences of training the model by minimizing Tweedie deviance.
Special attention is devoted to the Poisson regression case. 
The concept of Tweedie dominance is also introduced there, as a natural counterpart to
Bregman dominance, or forecast dominance discussed in Kruger and Ziegel (2020).
In Section 5, we propose a simple and powerful method to restore balance at both global and
local levels, based on the concept of autocalibration. 
Tweedie dominance between autocalibrated predictors reduces to the well-known convex
order, or stop-loss order with equal means that has been proposed to compare predictors
by Denuit et al. (2019) and Kruger and Ziegel (2020). 
A numerical study is provided in Section 6. The final Section 7 discusses the results and concludes.

\section{Context, definition and notation}

Let us now describe the notation used in this paper. We
consider a response $Y$ and a set of features $X_1,\ldots,X_p$ gathered in the vector $\Xvec\in\mathcal{X}$ (classically, $\mathcal{X}\subset\mathbb{R}^p$). 
In this paper, the response is typically the number of claims reported to the insurance
company by a given policyholder, the average claim severity or the total claim amount
in relation with this contract.
The dependence structure inside the random vector $(Y,X_1,\ldots,X_p)$
is exploited to extract the information contained in $\Xvec$ about $Y$.
In actuarial pricing, the aim is to evaluate the pure premium as accurately as possible. This means that
the target is the conditional expectation $\mu(\Xvec)=\Esp[Y|\Xvec]$ of the response $Y$ (claim
number or claim amount) given the available information $\Xvec$.
Henceforth, $\mu(\Xvec)$ is referred to as the true (pure) premium.
Notice that in some applications, $\mu(\Xvec)$ only refers to one component of the pure premium.
For instance, working in the frequency-severity decomposition of insurance losses,
$\mu(\Xvec)$ can be either the expected number of insured events or the expected claim size, or severity.

The function $\xvec\mapsto\mu(\xvec)=\Esp[Y|\Xvec=\xvec]$ is unknown to the actuary,
and may exhibit a complex behavior in $\xvec$. This is why this function is approximated
by a (working, or actual) premium $\xvec\mapsto\pi(\xvec)$ with a simpler structure.
When the analyst is working in the frequency-severity decomposition 
of insurance losses, $\pi(\xvec)$ targets the expected number of insured events or the
mean claim severity, separately. Once fitted on the training data set using an appropriate
learning procedure, this produces estimates $\widehat{\pi}(\xvec)$
for $\mu(\xvec)$, or fitted values.

The developments in this paper apply in any setting where a global balance is desirable,
that is, where it is important that the sum of estimates does not deviate too much from the
sum of actual observations at both the entire portfolio level and also more locally,
in meaningful classes of policyholders. The reason is obvious: the sum of the pure premiums must match the
claim total as accurately as possible so that the insurance company is able to indemnify all
third-parties and beneficiaries in execution of the contracts, without excess nor deficit,
by the very definition of pure premium (expense loadings and cost-of-capital
charges are added into the calculation at a later stage, when moving to commercial premiums). This naturally
translates into a global balance constraint: considering that the total claim figures are representative
of next-year's experience, it is important that the sum of fitted premiums $\widehat{\pi}(\xvec)$
match the sum of responses $Y$ taken as proxy for the total premium income (i.e., the sum of $\mu(\xvec)$),
as closely as possible. But local equilibrium is also essential to guarantee a competitive pricing.

The merits of a given pricing tool can be assessed using the pair
$\big(\mu(\Xvec),\widehat{\pi}(\Xvec)\big)$ so that we are back to the bivariate case
even if there were thousands of features comprised in $\Xvec$.
What really matters is the correlation between $\widehat{\pi}(\Xvec)$ and $\mu(\Xvec)$ but as $\mu(\Xvec)$ is unobserved
the actuary can only use its noisy version $Y$ to reveal the agreement of the true premium
$\mu(\Xvec)$ with its working counterpart $\widehat{\pi}(\Xvec)$. In insurance applications,
$\widehat{\pi}(\Xvec)$ is supposed to be used as a premium so that correlation is important, but it is also
essential that the sum of predictions $\widehat{\pi}(\Xvec)$ matches the sum of actual losses as closely as possible,
as explained before. This is expressed by the global balance condition and its local version.

To ease the exposition, we assume that predictor $\widehat{\pi}(\Xvec)$ under consideration,
as well as the conditional expectation $\mu(\Xvec)$ are continuous random variables
admitting probability density functions.
This is generally the case when there is at least one continuous feature contained in the
available information $\Xvec$ and the function $\widehat{\pi}$ is a continuously increasing function of a real score
built from $\Xvec$.
However, this rules out predictions based on discrete features only, as well as piecewise constant predictors, e.g., a single tree.
Indeed, then $\widehat{\pi}(\Xvec)$ takes only a limited number of values. As actuarial pricing is nowadays based on more sophisticated models
(trees being combined into random forests, for instance), this continuity assumption does not really restrict the generality of the
approach.

\section{Tweedie model and deviance}

In this paper, we assume that the response obeys a probability distribution belonging to the
Tweedie subclass of the Exponential Dispersion family. Precisely, this means that we
assume that the logarithm of the probability mass function for a discrete response, 
or of the probability density function for a continuous response, is of the
form $\ln f=(y\theta-a(\theta))/\phi$ up to a constant term, for some known dispersion parameter $\phi$ (that may include a weight)
and non-decreasing and convex cumulant function $a(\cdot)$ with $a'(\theta)=\Esp[Y]=\mu$.
The variance is then given by
$$
\Var[Y]=\phi a''(\theta)=\phi V(\mu)
$$
where the variance function $V(\cdot)$ corresponds to the second derivative of the cumulant function whose argument $\theta$
has been replaced in terms of $\mu$, that is, $\theta=(a')^{-1}(\mu)$. 

The Tweedie subclass corresponds to variance functions of the form $V(\mu)=\mu^\xi$ for some power parameter $\xi$.
Table \ref{TabTweedie} lists all Tweedie distributions.
Negative values of $\xi$ give continuous distributions on the whole real axis. 
For $0<\xi<1$, there is no Exponential Dispersion distribution with such variance function.
Only the cases $\xi \geq 1$ are thus interesting for applications in insurance.
In the remainder of this paper, we thus restrict our analysis to $\xi \geq 1$.

For $\xi\in(1,2)$, Tweedie distributions correspond to compound
Poisson-Gamma distributions, that is, compound Poisson sums with Gamma-distributed summands. Starting from
\begin{eqnarray*}
\frac{\mathrm{d}}{\mathrm{d}\mu}\ln f
&=&\frac{y-\mu}{\phi V(\mu)},
\end{eqnarray*}
we get with $V(\mu)=\mu^\xi$ that the probability density function over $(0,\infty)$
is given by
\begin{eqnarray*}
\ln f(y)&=&\int\frac{y-m}{\phi m^\xi}\mathrm{d}m\\
&=&\frac{1}{\phi}\left(y\frac{\mu^{1-\xi}}{1-\xi}-\frac{\mu^{2-\xi}}{2-\xi}\right)+\text{constant}
\end{eqnarray*}
with probability mass at zero $\exp\left(-\frac{\mu^{2-\xi}}{\phi(2-\xi)}\right)$.
Such compound Poisson-Gamma distributions can be used for
modeling annual claim amounts, having positive probability at zero and a continuous distribution
on the positive real numbers. This offers an alternative to the decomposition of total losses
into claim numbers and claim severities, using Poisson distribution for modeling claim counts and Gamma
distribution for claim severities. We refer the reader to Delong et al. (2021) for a through presentation
of Tweedie models.

\begin{table}
\begin{center} 
\begin{tabular}{lll}
\hline
 & Type & Name \\
\hline 
$\xi<0$ & Continuous & -  \\
$\xi=0$ & Continuous & Normal  \\
$0<\xi<1$ & Non existing & - \\
$\xi=1$ & Discrete & Poisson \\
$1<\xi<2$ & Mixed, non-negative & Compound Poisson sum\\
&&with Gamma-distributed terms  \\
$\xi=2$ & Continuous, positive & Gamma \\
$2<\xi<3$ & Continuous, positive & - \\
$\xi=3$ & Continuous, positive & Inverse Gaussian \\
$\xi>3$ & Continuous, positive & - \\
\hline
\end{tabular}
\end{center}
\caption{Tweedie distributions and corresponding power parameters.}
\label{TabTweedie}
\end{table}

Let $\widehat{\pi}$ be the estimated mean response $Y$ built from some training set (all formulas
in this paper are meant given this training set). 
The respective performances of competing models can then be assessed
on the basis of a validation set $\{(Y_i,\Xvec_i),\hspace{2mm}i=1,2,\ldots,n\}$, 
that has not been used to obtain $\widehat{\pi}$. 
Performances of $\widehat{\pi}$ are generally assessed with the help of
out-of-(training) sample deviance, also called predictive deviance and given by
$$
D_n(\xi,\widehat{\pi})
=\frac{1}{n}\sum_{i=1}^nL\big(Y_i,\widehat{\pi}(\Xvec_i)\big)
$$
where $L(\cdot,\cdot)$ is the loss function adopted to train the model.
If $n$ is large enough then we can resort to the limiting value
$$
D_n(\xi,\widehat{\pi})\to D(\xi,\widehat{\pi})=
\Esp\big[L\big(Y^{\text{new}},\widehat{\pi}(\Xvec^{\text{new}})\big)\big]
\text{ as }n\to\infty,
$$
where $(Y^{\text{new}},\Xvec^{\text{new}})$ is a new observation, independent of, 
and distributed as those $(Y_i,\Xvec_i)$ contained in the training set. 
In this paper, we compare models on the basis of the large-sample version
of the predictive deviance. This approach is meaningful in insurance applications where the
analyst is typically in a data-rich situation.

Henceforth, we compare models on the basis of the predictive Tweedie deviance
that reduces to
\begin{equation}
\label{deviance_loss_Tweedie}
D(\xi,\widehat{\pi})=\left\{
\begin{array}{l}
\Esp\left[\widehat{\pi}(\Xvec^\text{new}) - Y^\text{new} \ln \widehat{\pi}(\Xvec^\text{new})\right] \text{ for $\xi=1$}\\
 \\
\Esp\left[\ln\widehat{\pi}(\Xvec^\text{new}) + \frac{Y^\text{new}}{\widehat{\pi}(\Xvec^\text{new})}\right] \text{ for $\xi=2$}\\
 \\
\Esp\left[ \frac{\widehat{\pi}(\Xvec^\text{new})^{2-\xi}}{2-\xi} - \frac{Y^\text{new} \widehat{\pi}(\Xvec^\text{new})^{1-\xi}}{1-\xi} \right] \text{ for $\xi>1$ and $\xi\neq 2$}
\end{array}
\right. .
\end{equation}
In the remainder of the paper, we use \eqref{deviance_loss_Tweedie}
to assess the performances of a given predictor $\widehat{\pi}$.

\section{Tweedie dominance}

Tweedie dominance is defined as dominance for every Tweedie deviances \eqref{deviance_loss_Tweedie}. Precisely,
$\widehat{\pi}_2$ outperforms $\widehat{\pi}_1$ in terms of Tweedie dominance if the inequality
$D(\xi,\widehat{\pi}_2)\leq D(\xi,\widehat{\pi}_1)$ holds true for every power parameter $\xi\geq 1$.
Tweedie dominance appears to be a particular case of Bregman dominance, also called forecast dominance defined as dominance for every Bregman loss
function. We refer the interested reader to Kruger and Ziegel (2020) and the references therein
for an extensive presentation of this concept.
Tweedie dominance is thus a particular stochastic order relation used to
compare the performances of two estimators $\widehat{\pi}_1$ and $\widehat{\pi}_2$
for the conditional means. We refer the reader to Shaked and Shanthikumar (2007)
for a general presentation of stochastic order relations and to Denuit et al. (2005)
for applications to insurance.

The next result provides the actuary with a sufficient condition for a model to outperform a competitor
in terms of Tweedie dominance.

\begin{Proposition}
\label{PropBetterModel}
Define
\begin{equation}
\label{DefPsi}
\psi_\xi(\pi)=\left\{
\begin{array}{l}
\ln \pi \text{ for $\xi=2$}\\
 \\
\frac{\pi^{2-\xi}}{2-\xi} \text{ else}.
\end{array}
\right.
\end{equation}
Then, $\widehat{\pi}_2$ outperforms $\widehat{\pi}_1$ in terms of
Tweedie dominance if
\begin{equation}
\label{cond1}
\Esp\left[  \psi_\xi(\widehat{\pi}_1(\Xvec^\mathrm{new}))\right] \geq \Esp\left[  \psi_\xi(\widehat{\pi}_2(\Xvec^\mathrm{new}))\right]
\text{ for all }\xi\geq 1
\end{equation}
and 
\begin{equation}
\label{cond2}
\Esp\big[ Y^\mathrm{new} \Ind\left[ \widehat{\pi}_1(\Xvec^\mathrm{new}) \leq t\right] \big] 
\geq \Esp\big[ Y^\mathrm{new} \Ind\left[ \widehat{\pi}_2(\Xvec^\mathrm{new}) \leq t\right] \big] 
\quad \text{for all } t\geq 0.
\end{equation}
\end{Proposition}
\begin{proof}
For $\xi=1$, $\widehat{\pi}_2$ is superior to $\widehat{\pi}_1$ if
\begin{eqnarray}
\label{ineq_Poisson}
&&\Esp\left[\widehat{\pi}_1(\Xvec^\text{new})\right]-\Esp\left[Y^\text{new}\ln\widehat{\pi}_1(\Xvec^\text{new}) \right] \geq \Esp\left[\widehat{\pi}_2(\Xvec^\text{new})\right]-\Esp\left[Y^\text{new}\ln\widehat{\pi}_2(\Xvec^\text{new}) \right] \nonumber \\
&\Leftrightarrow&\Esp\left[\widehat{\pi}_1(\Xvec^\text{new})\right] - \Esp\left[\widehat{\pi}_2(\Xvec^\text{new})\right]+\Esp\left[Y^\text{new}\ln\widehat{\pi}_2(\Xvec^\text{new}) \right]-\Esp\left[Y^\text{new}\ln\widehat{\pi}_1(\Xvec^\text{new}) \right] \geq 0.\nonumber \\
\end{eqnarray}
Since the identity
\begin{eqnarray}
\Esp\left[Y^\text{new}\ln\widehat{\pi}(\Xvec^\text{new}) \right] 
&=& \int_0^\infty \Esp\big[ Y^\text{new} \Ind\left[ \ln \widehat{\pi}(\Xvec^\text{new}) > t \right] \big] \mathrm{d}t 
-  \int_{-\infty}^0 \Esp\big[ Y^\text{new} \Ind\left[ \ln \widehat{\pi}(\Xvec^\text{new}) \leq t \right] \big] 
\mathrm{d}t \nonumber \\
&=& \int_1^\infty \Esp\big[ Y^\text{new} \Ind\left[ \widehat{\pi}(\Xvec^\text{new}) > s \right] \big] \frac{1}{s} 
\mathrm{d}s 
-  \int_{0}^1 \Esp\big[ Y^\text{new} \Ind\left[ \widehat{\pi}(\Xvec^\text{new}) \leq s \right] \big] \frac{1}{s} 
\mathrm{d}s \nonumber 
\end{eqnarray}
holds true for any predictor $\widehat{\pi}$,
we have
\begin{eqnarray}
&&\Esp\left[Y^\text{new}\ln\widehat{\pi}_2(\Xvec^\text{new}) \right]-\Esp\left[Y^\text{new}\ln\widehat{\pi}_1(\Xvec^\text{new}) \right] \nonumber \\
&=&\int_1^\infty \big( \Esp\big[ Y^\text{new} \Ind\left[ \widehat{\pi}_2(\Xvec^\text{new}) > s \right] \big] 
- \Esp\big[ Y^\text{new} \Ind\left[ \widehat{\pi}_1(\Xvec^\text{new}) > s \right] \big] \big) 
\frac{1}{s} \mathrm{d}s \nonumber \\
&&- \int_{0}^1 \big( \Esp\big[ Y^\text{new} \Ind\left[ \widehat{\pi}_2(\Xvec^\text{new}) \leq s \right] \big] 
- \Esp\big[ Y^\text{new} \Ind\left[ \widehat{\pi}_1(\Xvec^\text{new}) \leq s \right] \big] \big) 
\frac{1}{s} \mathrm{d}s \nonumber \\
&=&\int_1^\infty \big( \Esp\big[ Y^\text{new} \left(1-\Ind\left[ \widehat{\pi}_2(\Xvec^\text{new}) \leq s \right] \right) \big] - \Esp\big[ Y^\text{new} \left(1-\Ind\left[ \widehat{\pi}_1(\Xvec^\text{new}) \leq s \right] \right)\big] \big) \frac{1}{s} \mathrm{d}s \nonumber \\
&&+ \int_{0}^1 \big( \Esp\big[ Y^\text{new} \Ind\left[ \widehat{\pi}_1(\Xvec^\text{new}) \leq s \right] \big] 
- \Esp\big[ Y^\text{new} \Ind\left[ \widehat{\pi}_2(\Xvec^\text{new}) \leq s \right] \big] \big) 
\frac{1}{s} \mathrm{d}s \nonumber \\
&=&\int_{0}^\infty \big( \Esp\big[ Y^\text{new} \Ind\left[ \widehat{\pi}_1(\Xvec^\text{new}) \leq s \right] \big] 
- \Esp\big[ Y^\text{new} \Ind\left[ \widehat{\pi}_2(\Xvec^\text{new}) \leq s \right] \big] \big) 
\frac{1}{s} \mathrm{d}s \nonumber \\
&\geq&0 \nonumber
\end{eqnarray}
by \eqref{cond2}. 
Hence, both conditions \eqref{cond1} and \eqref{cond2} ensure inequality \eqref{ineq_Poisson}. 

Turning to the case $\xi=2$, $\widehat{\pi}_2$ is superior to $\widehat{\pi}_1$ if
\begin{eqnarray}
\label{ineq_Gamma}
&&\Esp\left[\ln \widehat{\pi}_1(\Xvec^\text{new})\right]+\Esp\left[\frac{Y^\text{new}}{\widehat{\pi}_1(\Xvec^\text{new})} \right] \geq \Esp\left[\ln \widehat{\pi}_2(\Xvec^\text{new})\right]+\Esp\left[\frac{Y^\text{new}}{\widehat{\pi}_2(\Xvec^\text{new})} \right] \nonumber \\
&\Leftrightarrow&\Esp\left[\ln \widehat{\pi}_1(\Xvec^\text{new})\right]- \Esp\left[\ln \widehat{\pi}_2(\Xvec^\text{new})\right]+\Esp\left[\frac{Y^\text{new}}{\widehat{\pi}_1(\Xvec^\text{new})} \right] -\Esp\left[\frac{Y^\text{new}}{\widehat{\pi}_2(\Xvec^\text{new})} \right]  \geq 0.\nonumber \\
\end{eqnarray}
Hence, it suffices to notice that \eqref{cond2} implies $\Esp\left[\frac{Y^\text{new}}{\widehat{\pi}_1(\Xvec^\text{new})} \right] -\Esp\left[\frac{Y^\text{new}}{\widehat{\pi}_2(\Xvec^\text{new})} \right]  \geq 0$ since the identity
\begin{eqnarray}
\Esp\left[\frac{Y^\text{new}}{\widehat{\pi}(\Xvec^\text{new})} \right] 
&=& \int_0^\infty \Esp\left[Y^\text{new} \Ind\left[ \frac{1}{\widehat{\pi}(\Xvec^\text{new})} \geq t\right] \right] 
\mathrm{d}t \nonumber \\
&=& \int_0^\infty \Esp\left[Y^\text{new} \Ind\left[ \widehat{\pi}(\Xvec^\text{new}) \leq \frac{1}{t}\right] \right] 
\mathrm{d}t \nonumber \\
&=& \int_0^\infty \Esp\left[Y^\text{new} \Ind\left[ \widehat{\pi}(\Xvec^\text{new}) \leq s\right] \right] 
\frac{1}{s^2}\mathrm{d}s. \nonumber
\end{eqnarray}
is valid for every predictor $\widehat{\pi}$.

Finally, in the remaining cases, i.e. $\xi>1$ and $\xi \neq 2$, $\widehat{\pi}_2$ is superior to $\widehat{\pi}_1$ if
\begin{eqnarray}
\label{ineq_OtherCases}
&&\Esp\left[ \frac{\widehat{\pi}_1(\Xvec^\text{new})^{2-\xi}}{2-\xi} \right] + \frac{1}{\xi-1} \Esp\left[ \frac{Y^\text{new} }{\widehat{\pi}_1(\Xvec^\text{new})^{\xi-1}} \right] 
\geq \Esp\left[ \frac{\widehat{\pi}_2(\Xvec^\text{new})^{2-\xi}}{2-\xi} \right] + \frac{1}{\xi-1} \Esp\left[ \frac{Y^\text{new} }{\widehat{\pi}_2(\Xvec^\text{new})^{\xi-1}} \right]  \nonumber \\
&\Leftrightarrow& 
\Esp\left[ \frac{\widehat{\pi}_1(\Xvec^\text{new})^{2-\xi}}{2-\xi} \right] - \Esp\left[ \frac{\widehat{\pi}_2(\Xvec^\text{new})^{2-\xi}}{2-\xi} \right] + \frac{1}{\xi-1} \left(\Esp\left[ \frac{Y^\text{new} }{\widehat{\pi}_1(\Xvec^\text{new})^{\xi-1}} \right] - \Esp\left[ \frac{Y^\text{new} }{\widehat{\pi}_2(\Xvec^\text{new})^{\xi-1}} \right] \right) \geq 0. \nonumber \\
\nonumber \\
\end{eqnarray}
Whatever the predictor $\widehat{\pi}$, we can write 
\begin{eqnarray}
\Esp\left[\frac{Y^\text{new}}{\widehat{\pi}(\Xvec^\text{new})^{\xi-1}} \right] 
&=& \int_0^\infty \Esp\left[Y^\text{new} \Ind\left[ \widehat{\pi}(\Xvec^\text{new})^{1-\xi} \geq t\right] \right] \mathrm{d}t \nonumber \\
&=& \int_0^\infty \Esp\left[Y^\text{new} \Ind\left[ \widehat{\pi}(\Xvec^\text{new}) \leq \frac{1}{t^{\frac{1}{\xi-1}}}\right] \right] \mathrm{d}t \nonumber \\
&=& \int_0^\infty \Esp\big[Y^\text{new} \Ind\left[ \widehat{\pi}(\Xvec^\text{new}) \leq s\right] \big] \frac{\xi-1}{s^\xi}\mathrm{d}s, \nonumber
\end{eqnarray}
The announced result then follows from 
\begin{eqnarray}
&&\Esp\left[\frac{Y^\text{new}}{\widehat{\pi}_1(\Xvec^\text{new})^{\xi-1}} \right] -\Esp\left[\frac{Y^\text{new}}{\widehat{\pi}_2(\Xvec^\text{new})^{\xi-1}} \right] \nonumber \\
&&= \int_0^\infty \Big(\Esp\big[Y^\text{new} \Ind\left[ \widehat{\pi}_1(\Xvec^\text{new}) \leq s\right] \big]-\Esp\big[Y^\text{new} \Ind\left[ \widehat{\pi}_2(\Xvec^\text{new}) \leq s\right] \big] \Big) \frac{\xi-1}{s^\xi}\mathrm{d}s. \nonumber 
\end{eqnarray}
This ends the proof.
\end{proof}

Instead of Tweedie dominance, the actuary could select a specific parameter $\xi$, only.
This is typically the case with $\xi=1$ (Poisson regression for counts),
$\xi=2$ or $\xi=3$ (Gamma or Inverse Gaussian regression for claim severities).
In this case, condition \eqref{cond1} in Proposition \ref{PropBetterModel} is imposed
only for that specific valued of $\xi$.

The main ingredient of the proof of Proposition \ref{PropBetterModel}
is the integral of the difference of functions in \eqref{cond2} weighted by
$$
\frac{\xi-1+\Ind[\xi=1]}{s^\xi}=\frac{\xi-1+\Ind[\xi=1]}{V(s)}
$$
where $V(\cdot)$ is the variance function associated with the response distribution.
Weights thus appear to be inversely proportional to the variability.
Under the Tweedie variance function, $V(s)$ is a power of $s$ so that
smaller weights are assigned to the differences
at larger values of the response.

Proposition \ref{PropBetterModel} shows that improving a predictor $\widehat{\pi}$ can be achieved by
\begin{itemize}
\item[(i)]
increasing the overall bias measured on a modified scale induced by
the auxiliary function $\psi_\xi$ defined in \eqref{DefPsi}.
This may act against the conservation of observed totals.
\item[(ii)]
increasing the dependence between $\widehat{\pi}$ and the response,
so to decrease
\begin{eqnarray*}
\Esp\Big[Y^{\text{new}}\Ind\big[\widehat{\pi}(\Xvec^{\text{new}})\leq t\big]\Big]
&=&
\Cov\Big[Y^{\text{new}},\Ind\big[\widehat{\pi}(\Xvec^{\text{new}})\leq t\big]\Big]
+\Esp\Big[Y^{\text{new}}\Big]\Prob\Big[\widehat{\pi}(\Xvec^{\text{new}})\leq t\Big]
\\
&=&
\Esp\Big[Y^{\text{new}}\Big|\widehat{\pi}(\Xvec^{\text{new}})\leq t\Big]
\Prob\Big[\widehat{\pi}(\Xvec^{\text{new}})\leq t\Big].
\end{eqnarray*}
The latter lower partial moment essentially depends on the correlation structure of 
the pair $\big(Y^{\text{new}},\widehat{\pi}(\Xvec^{\text{new}})\big)$.
Notice that the quantities appearing in the decomposition
above are closely related to expectation dependence as defined by Wright (1987).

In an insurance ratemaking context, the lower partial moment can be interpreted as best-profile premium income
\begin{eqnarray*}
\Esp\Big[Y^{\text{new}}\Ind\big[\widehat{\pi}(\Xvec^{\text{new}})\leq t\big]\Big]
&=&
\Esp\Big[\mu(\Xvec^{\text{new}})\Ind\big[\widehat{\pi}(\Xvec^{\text{new}})\leq t\big]\Big]
\\
&\approx&
\frac{1}{n}\sum_{i|\widehat{\pi}(\Xvec_i)\leq t}\mu(\Xvec_i).
\end{eqnarray*}
Here, $\sum_{i|\widehat{\pi}(\Xvec_i)\leq t}\mu(\Xvec_i)$ is the true premium income for the
sub-portfolio formed by gathering all policyholders with predicted premium at most equal to $t$.
These policyholders exhibit the best risk profiles according to the candidate premium $\widehat{\pi}$.
\end{itemize}

Of course, there is a trade-off between these two goals when computing
$D(\xi,\widehat{\pi})$.
If the model is very flexible then it can produce a predictor that correlates a lot with the response and the
lower partial moment can be decreased to a large extent compared to models imposing a rigid form for the predictor
(like GLMs). There is in fact a fundamental trade-off between \eqref{cond1} and \eqref{cond2}. Since $\psi_\xi$ is increasing, \eqref{cond1} favors small values of
$\widehat{\pi}_2$. By contrast, \eqref{cond2} requires large values of $\widehat{\pi}_2$
so that the indicator function on the right-hand side is 0. There is thus no guarantee that balance holds on a particular data set: if one component dominates, we may end up
with candidate premiums that are either too small or too large compared to observed responses.

\begin{Example}[Poisson regression]

Poisson deviance is by far the most widely used one in insurance applications. 
It applies for instance to claim counts in property and casualty insurance, 
death counts in life insurance, and numbers of transitions in health insurance. 
We assume that we deal with a response $Y$ obeying the Poisson distribution. 
A vector $\Xvec$ of features
is available to predict the mean response $\mu(\Xvec)$. To ease the presentation, we assume unit exposures.

Considering \eqref{DefPsi}-\eqref{cond1} with $\xi=1$,
the bias is thus measured on the response scale.
Condition \eqref{cond1} then favors $\widehat{\pi}_2$ if
$$
\Esp[\mu(\Xvec^{\text{new}})]-\Esp[\widehat{\pi}_2(\Xvec^{\text{new}})]\geq\Esp[\mu(\Xvec^{\text{new}})]-\Esp[\widehat{\pi}_1(\Xvec^{\text{new}})]
$$
or, equivalently,
$$
\Esp[Y^{\text{new}}]-\Esp[\widehat{\pi}_2(\Xvec^{\text{new}})]\geq\Esp[Y^{\text{new}}]-\Esp[\widehat{\pi}_1(\Xvec^{\text{new}})].
$$
A larger bias may thus be advantageous, depending on the fundamental trade-off
between \eqref{cond1} and \eqref{cond2} explained above. As a consequence,
adopting Poisson deviance as loss function outside GLMs may create a total premium income gap
\begin{eqnarray*}
\Esp\Big[Y^{\text{new}}\Big]-\Esp\Big[\widehat{\pi}(\Xvec^{\text{new}})\Big]
&=&
\Esp\Big[\mu(\Xvec^{\text{new}})\Big]-\Esp\Big[\widehat{\pi}(\Xvec^{\text{new}})\Big]
\\
&\approx&
\frac{1}{n}\sum_{i=1}^n\mu(\Xvec_i)-\frac{1}{n}\sum_{i=1}^n\widehat{\pi}(\Xvec_i)
\end{eqnarray*}
where $\sum_{i=1}^n\mu(\Xvec_i)$ is the true premium income for the validation set
whereas $\sum_{i=1}^n\widehat{\pi}(\Xvec_i)$ is the one obtained by adopting predictor $\widehat{\pi}$
for premium calculation. This gap may become quite large with highly flexible models such as Neural Networks or boosting.

\end{Example}

\section{Restoring balance at global and local scales}
\label{SecAutoCal}

\subsection{Autocalibration}

Recall that a predictor $\widehat{\pi}$ is said to be autocalibrated if 
$\widehat{\pi}(\Xvec)=\Esp[Y|\widehat{\pi}(\Xvec)]$. We refer the reader to
Kruger and Ziegel (2020) for a general presentation of this concept. By Jensen inequality,
autocalibration thus ensures that
$$
\Esp[g(\widehat{\pi}(\Xvec))]\leq\Esp[g(Y)]
\text{ for every convex function }g,
$$
or equivalently, that
$$
\Esp[\widehat{\pi}(\Xvec)]=\Esp[Y]\text{ and }
\Esp[(\widehat{\pi}(\Xvec)-t)_+]\leq\Esp[(Y-t)_+]\text{ for all }t\geq 0.
$$
These inequalities correspond to the convex order between $\widehat{\pi}(\Xvec)$
and $Y$. Thus, autocalibration implies that the predictor is less variable than
the response, in the sense of the convex order. 

Under mild technical requirement,
a simple way to restore global balance consists in switching from $\widehat{\pi}$ to its
balance-corrected version $\widehat{\pi}_{\text{BC}}$ defined as
$$
\widehat{\pi}_{\text{BC}}(\Xvec)=\Esp[Y|\widehat{\pi}(\Xvec)]
$$
that averages to $\Esp[Y]$, as shown in the next result.

\begin{Property}
If $s\mapsto \Esp[Y|\widehat{\pi}(\Xvec)=s]$ is continuously increasing then
the balance-corrected version $\widehat{\pi}_{\mathrm{BC}}$ of the candidate premium $\widehat{\pi}$
satisfies the autocalibration property.
\end{Property}
\begin{proof}
If $s\mapsto \Esp[Y|\widehat{\pi}(\Xvec)=s]$ is continuously increasing, that is,
if $\widehat{\pi}_{\text{BC}}(\Xvec)$ is continuously increasing in $\widehat{\pi}(\Xvec)$, then
Lemma 2.2 in Shaked et al. (2012) allows us to write
$$
\Esp[Y|\widehat{\pi}(\Xvec)]=\Esp[Y|\widehat{\pi}_{\text{BC}}(\Xvec)]
=\widehat{\pi}_{\text{BC}}(\Xvec)
$$
so that the resulting $\widehat{\pi}_{\text{BC}}(\Xvec)$ is indeed autocalibrated.
This ends the proof.
\end{proof}

In insurance applications, autocalibration
induces local balance and imposes financial equilibrium not only at portfolio level but also
in any sufficiently large sub-portfolio. This concept thus appears to be particularly appealing
in a ratemaking context.

\subsection{Autocalibrating a given predictor}

We can restore global balance, or unbias
the predictor by reconciling the predicted and observed total on the training set. In this way,
we recover the global balance property imposed after the seminal work by Bailey and Simon (1960).
However, this does not ensure that balance holds locally. Indeed, global balance is only one
of the GLM likelihood equations under canonical link, corresponding to the intercept. 
In order to extend the other
likelihood equations imposed in the GLM setting to general machine learning procedures,
we need to mimic the way local GLM proceeds for fitting, by defining meaningful neighborhoods for
statistical learning. 

To this end, we work under canonical link function with data points $(Y_i,e_i,\widehat{\pi}(\xvec_i))$ for some
relevant exposure $e_i$. An intuitively acceptable solution would consist
in imposing marginal constrains on local neighborhoods defined by mean of $\widehat{\pi}$.
This allows for some local transfers of claims and premiums from
neighboring policyholders and so implements local balance conditions in sub-portfolios 
corresponding to these neighborhoods. This is in essence
the local GLM approach (see Loader, 1999, for a detailed account) that allows 
the actuary to maintain the relationship with MMT.

In order to obtain an autocalibrated version $\widehat{\pi}_{\text{BC}}$ of $\widehat{\pi}$,
let us consider a specific risk profile $\xvec$.
A weight $\nu_i(\widehat{\pi}(\xvec))$ is assigned to each $(Y_i,e_i,\widehat{\pi}(\xvec_i))$,
$i=1,\ldots,n$, computed from some weight function $\nu(\cdot)$ chosen to be continuous, 
symmetric, peaked at 0 and defined on $[-1,1]$.
These weights depend on the relative distance of $\widehat{\pi}(\xvec_i)$ with respect
to $\widehat{\pi}(\xvec)$.
A common choice for $\nu(\cdot)$ is the tricube weight function but several alternatives are available, 
including rectangular (or uniform), Gaussian or Epanechnikov, for instance. 
Here, $\nu_i(\widehat{\pi}(\xvec))$ is larger for policyholders $i$ such that
$\widehat{\pi}(\xvec_i)$ is close to $\widehat{\pi}(\xvec)$ and decreases when
$\widehat{\pi}(\xvec_i)$ gets far away from $\widehat{\pi}(\xvec)$.

The local GLM likelihood equation
$$
\sum_{i=1}^n\nu_i(\widehat{\pi}(\xvec))y_i
=\sum_{i=1}^n\nu_i(\widehat{\pi}(\xvec))e_i\widehat{\pi}_{\text{BC}}(\xvec)
$$
thus matches MMT constraints:
smoothing is ensured by transferring part of the experience at neighboring $\widehat{\pi}$
values to obtain $\widehat{\pi}_{\text{BC}}$.
A local constant, or intercept-only GLM thus provides the appropriate fitting procedure
when balance must be respected.
Opting for a rectangular weight function complies with MMT: the weights $\nu_i$
are constant within the smoothing window and zero otherwise so that the sums reduce
to observations comprised within this window and the uniform weights factor out.

Our main message here is thus that a local, intercept-only GLM with rectangular weights implements local balance,
or MMT in a second step, within sub-portfolios gathering policyholders with about the same predicted value according to the first step. 
The rectangular weight function involved in the statistical
procedure optimally transfers part of claim experience between neighboring policyholders. This approach implements
smoothness from a statistical point of view while remaining fully transparent and understandable.
Indeed, local averaging can just be seen as an application of the mutuality principle at the heart of insurance.

\section{Numerical illustrations}

In this section, we first
consider the \texttt{freMTPL2freq} dataset, from the \texttt{CASDataset} R package of Charpentier (2014). The variable of interest is annual claim frequency. Precisely, we run Poisson regression on response \texttt{ClaimNb}, with exposure \texttt{Exposure} and various explanatory features (continuous and categorical). Then we use the \texttt{freMTPL2sev} dataset, with the severity, and use a Tweedie model to get an estimation of total annual losses.

\subsection{Claim frequency}

Four models are considered : a generalized linear model $\widehat{\pi}^{\text{\sffamily glm}}$, a generalized additive model $\widehat{\pi}^{\text{\sffamily gam}}$ where continuous features are transformed nonlinearly using spline functions, a boosting algorithm $\widehat{\pi}^{\text{\sffamily bst}}$ and a Neural Network model $\widehat{\pi}^{\text{\sffamily nn}}$. For the boosting models, we use the \texttt{h2o} package, with \texttt{h2o.gbm}. For the Neural Network approach, we use combined actuarial neural net (CANN\footnote{see {\texttt{https://www.kaggle.com/floser/glm-neural-nets-and-xgboost-for-insurance-pricing}} for a discussion and the codes used here.}), from Schelldorfer and W\"uthrich (2019), using Keras. Notice that in the latter case, even with the same seed, outputs of the Neural Nets procedure change.

In order to model our counting variable, we set \texttt{distribution = "poisson"} and \texttt{offset\_column = "Exposure"}. For the boosting inference procedure, we use \texttt{ntrees = 30} (the impact of the number of iterations will be discussed later on) and \texttt{nfolds = 5}\footnote{The codes used in this section can be found on the github repository \\
\texttt{https://github.com/freakonometrics/autocalibration}}. From the initial dataset, with 678,013 rows, 70\% are randomly chosen for our training dataset (474,609 rows) and used to construct $\widehat{\pi}$, and the remaining 30\% (203,404 rows) are used as a validation dataset, to draw the following graphs.

On Figure \ref{fig:CASDataset:hist:1}, we can visualise the distribution of predictions $\widehat{\pi}(\boldsymbol{x}_i)$. The average value $\overline{\pi}$ for the four models, on the training dataset, is given on the left of Table \ref{tab:CASDataset:mean} (additional information is also given there, namely quantiles). As a benchmark, if we run a simple Poisson regression of the intercept only, we obtain $\widehat{\beta}_0=-2.2911$, corresponding to a baseline prediction $\overline{\pi}=0.10115$. Observe that the boosting procedure globally underestimates claims frequency, while neural nets globally overestimate.

\begin{figure}[!h]
    \centering
    \includegraphics[width=\textwidth]{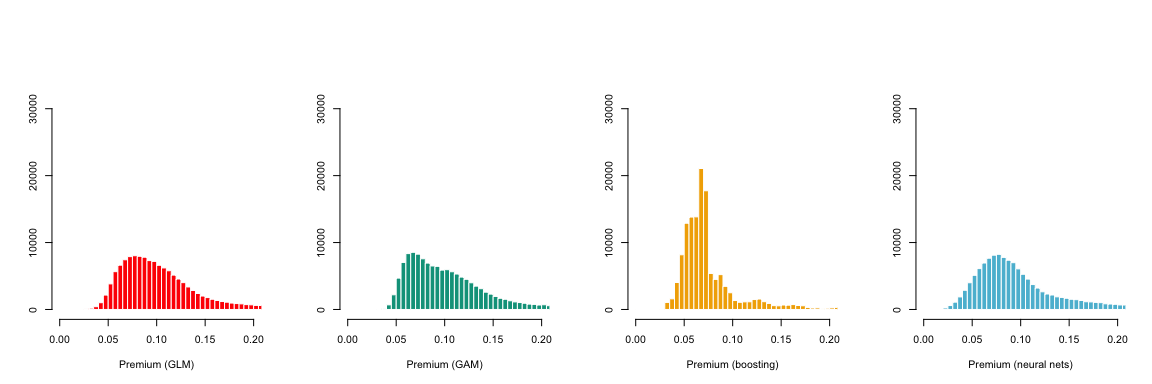}
    \caption{Histogram of $\{\widehat{\pi}(\boldsymbol{x}_1),\cdots,\widehat{\pi}(\boldsymbol{x}_n)\}$ on the validation dataset. In this section, \includegraphics[width=.2cm]{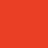} corresponds to the standard Poisson regression (GLM), \includegraphics[width=.2cm]{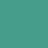} corresponds to the smooth additive regression (GAM), \includegraphics[width=.2cm]{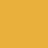} is the boosting model and \includegraphics[width=.2cm]{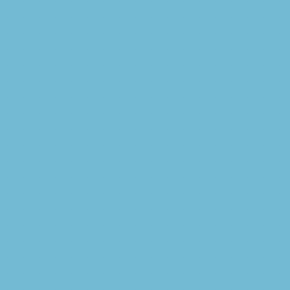} is the Neural Network.}
    \label{fig:CASDataset:hist:1}
\end{figure}

\begin{table}[!h]
    \centering
    \begin{tabular}{|c|cccc|}\hline
         &  $\widehat{\pi}^{\text{\sffamily glm}}$ & $\widehat{\pi}^{\text{\sffamily gam}}$ & $\widehat{\pi}^{\text{\sffamily bst}}$ & $\widehat{\pi}^{\text{\sffamily nn}}$  \\\hline
    average  $\overline{\pi}$   & 0.1091 & 0.1091 & 0.0821 & 0.1230 \\
    10\% quantile & 0.0602 &0.0592&0.0494 & 0.0529\\
    90\% quantile & 0.1688 & 0.1729 & 0.1264 & 0.2051 \\\hline
    \end{tabular}~~~\begin{tabular}{|c|cccc|}\hline
         &  $\widehat{\pi}_{BC}^{\text{\sffamily glm}}$ & $\widehat{\pi}_{BC}^{\text{\sffamily gam}}$ & $\widehat{\pi}_{BC}^{\text{\sffamily bst}}$& $\widehat{\pi}_{BC}^{\text{\sffamily nn}}$  \\\hline
    average  $\overline{\pi}$   & 0.1051 & 0.1059 & 0.1028 & 0.1051 \\
    10\% quantile & 0.0573 &0.0570&0.0518 & 0.0499\\
    90\% quantile & 0.1687 & 0.1806 & 0.1711 & 0.1776  \\\hline
    \end{tabular}
    \caption{Summary statistics on $\{\widehat{\pi}(\boldsymbol{x}_1),\cdots,\widehat{\pi}(\boldsymbol{x}_n)\}$, on the validation dataset (assuming an exposure of $1$ to provide annualized predictions), for $\pi$ on the left, and the corrected version $\pi_{BC}$ on the right.}
    \label{tab:CASDataset:mean}
\end{table}

Figure \ref{fig:CASDataset:2} displays $s\mapsto \Esp[Y|\widehat{\pi}(\boldsymbol{X})=s]$ when $s\in[0,0.2]$. For the Generalized Linear Model $\Esp[Y|\widehat{\pi}(\boldsymbol{X})=s]\sim s$, while $\Esp[Y|\widehat{\pi}(\boldsymbol{X})=s]> s$ for the boosting model. 
The positive bias we observe on $\widehat{\pi}^{\text{\sffamily bst}}$ means that this model {\em underestimates} the true price of the risk almost everywhere; similarly, the negative bias we observe on $\widehat{\pi}^{\text{\sffamily nn}}$ almost eveywhere means that this model {\em overestimates} the true price. As shown on Table \ref{tab:CASDataset:mean}, the bias correction step solves this issue.

\begin{figure}[!h]
    \centering
    \includegraphics[width=\textwidth]{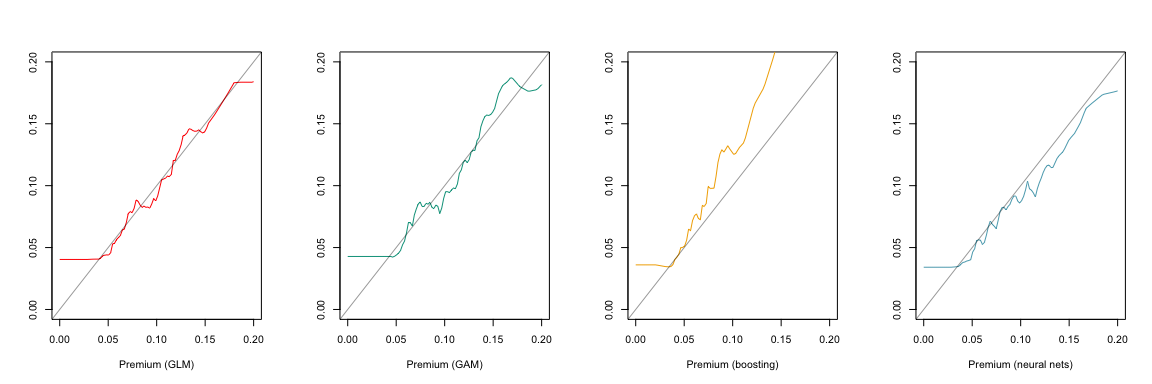}
    \caption{Evolution of $s\mapsto \Esp[Y|\widehat{\pi}(\boldsymbol{X})=s]$ (the thin straight line corresponds to $\mu=\widehat{\pi}$). }
    \label{fig:CASDataset:2}
\end{figure}

On Figure \ref{fig:CASDataset:pp}, we can compare $\widehat{\pi}$ and $\widehat{\pi}_{\text{BC}}$: on top, we have the QQ plot of $\widehat{\pi}_{\text{BC}}$ against $\widehat{\pi}$
and on the bottom, a scatterplot of $\{\widehat{\pi}(\boldsymbol{x}_i),\widehat{\pi}_{\text{BC}}(\boldsymbol{x}_i)\}$ (for a subset of the validation database). The distribution of $\widehat{\pi}_{\text{BC}}$ can be visualized on Figure \ref{fig:CASDataset:hist:2}.

\begin{figure}[!h]
    \centering
    \includegraphics[width=\textwidth]{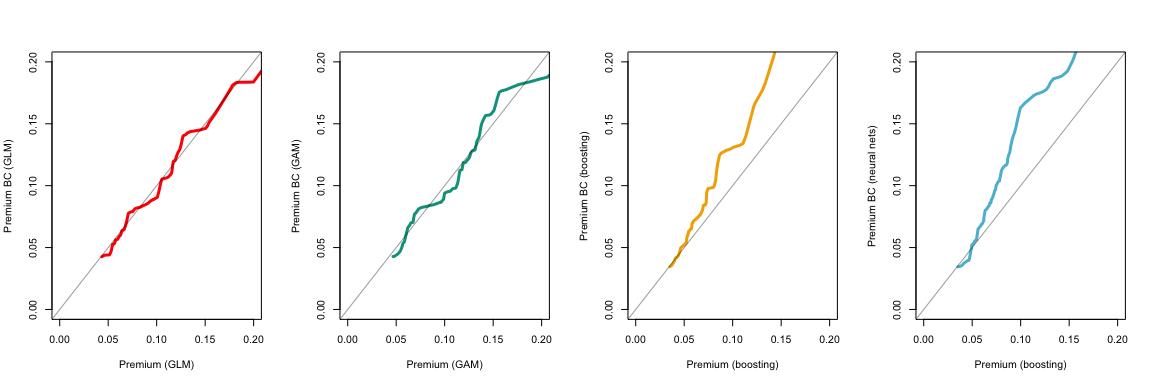}
    \includegraphics[width=\textwidth]{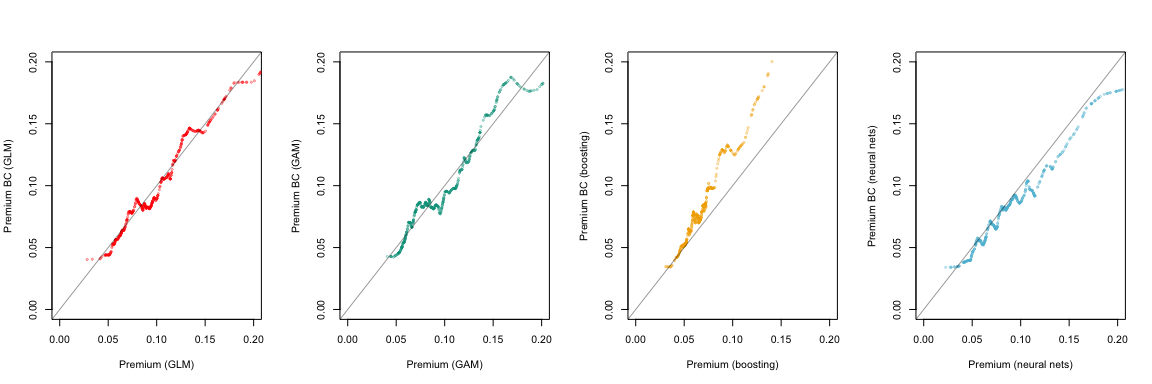}
    \caption{QQ plot of $\widehat{\pi}_{\text{BC}}$ against $\widehat{\pi}$ (plain line), on top, and scatterplot a subset of points $\{\widehat{\pi}(\boldsymbol{x}_i),\widehat{\pi}_{\text{BC}}(\boldsymbol{x}_i)\}$ from the validation dataset, below.}
    \label{fig:CASDataset:pp}
\end{figure}

\begin{figure}[!h]
    \centering
    \includegraphics[width=\textwidth]{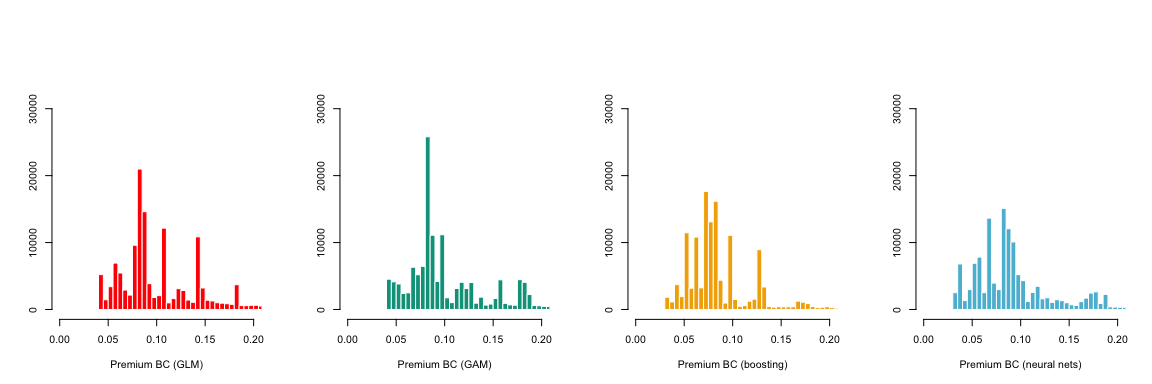}
    \caption{Histogram of $\{\widehat{\pi}_{\text{BC}}(\boldsymbol{x}_1),\cdots,\widehat{\pi}_{\text{BC}}(\boldsymbol{x}_n)\}$ on the validation dataset.}
    \label{fig:CASDataset:hist:2}
\end{figure}

The specification of bandwidth $h(s)$ is discussed in Loader (1999, Section 2.2.1). For a nearest neighbor bandwidth, compute all distances $|s-s_i|$ between the fitting point $s$ and the data points $s_i$, we choose $h(x)$ to be the $k$th smallest distance, where $k=[n\alpha]$. In the \texttt{R} function \texttt{locfit}, when \texttt{alpha} is given as a single number, it represents a nearest neighbor fraction (the default smoothing parameter is $\alpha=70\%$). But a second component can be added, $\alpha=(\alpha_0,\alpha_1)$. That second component represents a constant bandwidth, and $h(s)$ will be computed as follows: as previously, $k=[n\alpha_0]$, and if $d_{(i)}$ represents the ordered statistics of $d_i=|s-s_i|$, $h(s)=\max\{d_{(k)},\alpha_1\}$. The default value in R is \texttt{alpha=c(0.7,0)}.

As we can see on Table \ref{tab:CASDataset:mean}, the boosting algorithm was globally underestimating the average premium. Using a local regression can actually lower the bias. On Figure \ref{fig:CASDataset:alpha:evol:boost}, we used 60\% of the original dataset to train our three models (training dataset), then the local regression was performed on 20\% of the dataset (smoothing dataset) and finally various quantities (bias and empirical loss) were computed on the remaining 20\% (validation dataset). On the left, we can visualize the evolution of the bias $\displaystyle{\frac{1}{n}\sum_{i=1}^n(\widehat{\pi}^{\text{\sffamily bst}}_{\text{BC}}(\boldsymbol{x}_i)-y_i)}$ as a function of $\alpha$. Here, smoothing has no impact on the overall bias of GLM and GAM (here the two models have a (small) positive bias on the validation dataset). Smoothing has an impact on the correction of $\widehat{\pi}^{\text{\sffamily bst}}$. We can observe that a small $\alpha$ can lead to a much smaller bias (possibly null). Note that we used $\alpha=5\%$ for previous graphs of this section. In the middle of Figure \ref{fig:CASDataset:alpha:evol:boost}, we can visualize the evolution of the empirical Poisson loss $\displaystyle{\frac{1}{n}\sum_{i=1}^n(\widehat{\pi}^{\text{\sffamily bst}}_{\text{BC}}(\boldsymbol{x}_i)-y_i\ln\widehat{\pi}^{\text{\sffamily bst}}_{\text{BC}}(\boldsymbol{x}_i) )}$. Note that again, a small $\alpha$ leads to a smaller loss. It is also interesting to see that similar performances are obtained in case
of overfitting. Thus, autocalibration also corrects for overfitting by locally averaging the initial predictors.

\begin{figure}[!h]
    \centering
    \includegraphics[width=\textwidth]{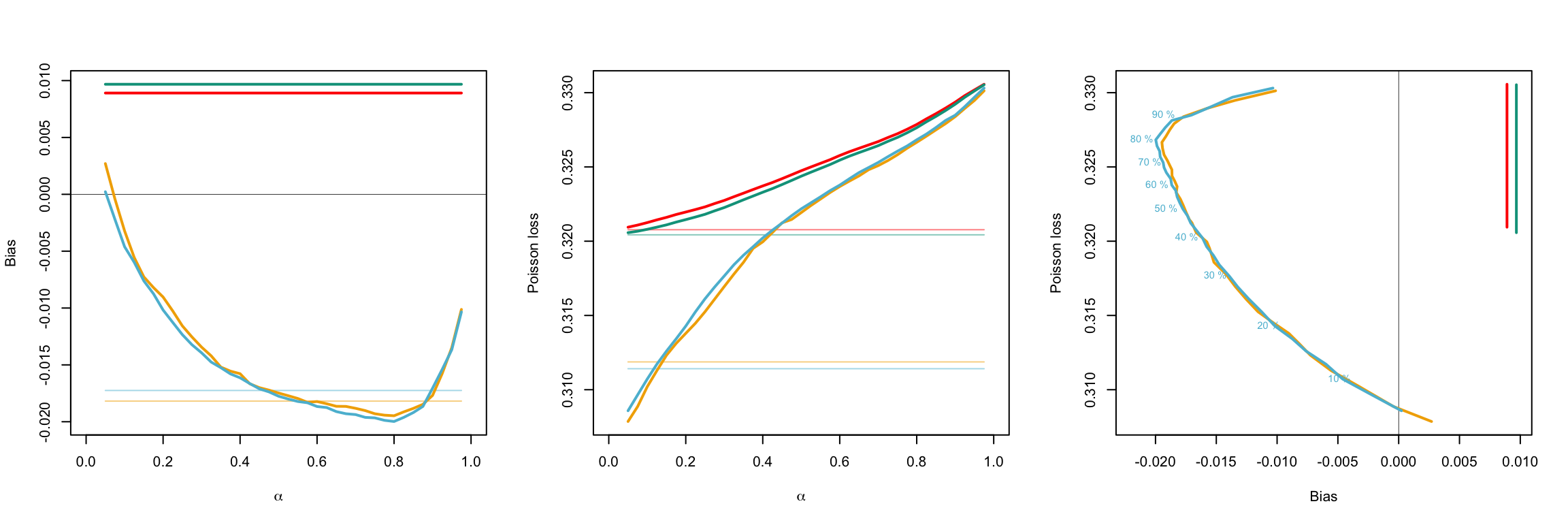}
    \caption{Evolution of the bias as a function a $\alpha$, from $2.5\%$ to $97.5\%$, on the left, of the Poisson loss in the center, and joint scatterplot of bias and empirical loss on the right. Horizontal light lines correspond to original estimates $\widehat{\pi}$ while strong curves are $\widehat{\pi}_{\text{BC}}$ for various smoothing parameter $\alpha$. As previously \includegraphics[width=.2cm]{figures/couleur-rouge.png} corresponds to the standard Poisson regression (GLM), \includegraphics[width=.2cm]{figures/couleur-vert.png} corresponds to the smooth additive regression (GAM), while \includegraphics[width=.2cm]{figures/couleur-jaune.png} is the boosting model with 30 trees (as previously) and \includegraphics[width=.2cm]{figures/couleur-bleu.png} is the boosting model with 1000 trees (overfitting).}
    \label{fig:CASDataset:alpha:evol:boost}
\end{figure}


As we have seen when moving from $\widehat{\pi}$ to $\widehat{\pi}_{\text{BC}}$ the distribution of premiums changed substantially (see Figures \ref{fig:CASDataset:hist:1} and \ref{fig:CASDataset:hist:2}) and the transformation was not (perfectly) monotonic (see Figure \ref{fig:CASDataset:pp}). Nevertheless, the rank correlation between $\widehat{\pi}$ and $\widehat{\pi}_{\text{BC}}$ is rather high (0.99 for all models), as reported in Figure \ref{fig:CASDataset:correl}.

\begin{figure}[!h]
    \centering
    \includegraphics[width=.5\textwidth]{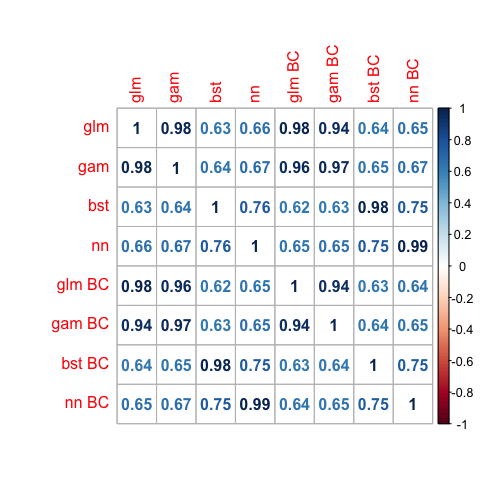}
    \caption{Spearman's rank correlation between $\widehat{\pi}^{\text{\sffamily glm}}$, $\widehat{\pi}^{\text{\sffamily gam}}$, $\widehat{\pi}^{\text{\sffamily bst}}$, $\widehat{\pi}^{\text{\sffamily nn}}$ and their autocalibrated counterparts, $\widehat{\pi}_{\text{BC}}^{\text{\sffamily glm}}$, $\widehat{\pi}_{\text{BC}}^{\text{\sffamily gam}}$, $\widehat{\pi}_{\text{BC}}^{\text{\sffamily bst}}$ and $\widehat{\pi}_{\text{BC}}^{\text{\sffamily nn}}$, on the validation dataset}
    \label{fig:CASDataset:correl}
\end{figure}

Finally, on Figure \ref{fig:CASDataset:CC:1} we can visualize the concentration curves of the three models, against $\widehat{\pi}^{\text{\sffamily gam}}_{\text{BC}}$.
We refer the reader to Denuit et al. (2019) for a thorough presentation of this diagnostic tool. Concentration curves clearly demonstrate here the improved performances of the autocalibrated version of the predictor.

\begin{figure}[!h]
    \centering
    \includegraphics[scale=0.2]{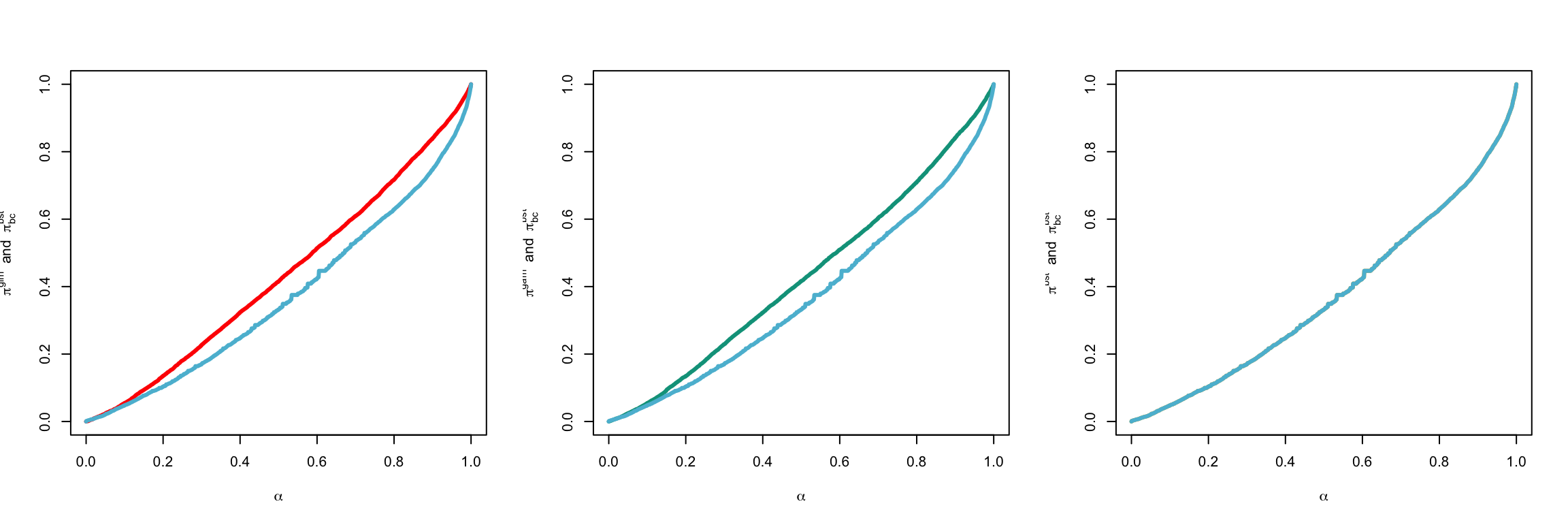}
    
    \includegraphics[scale=0.2]{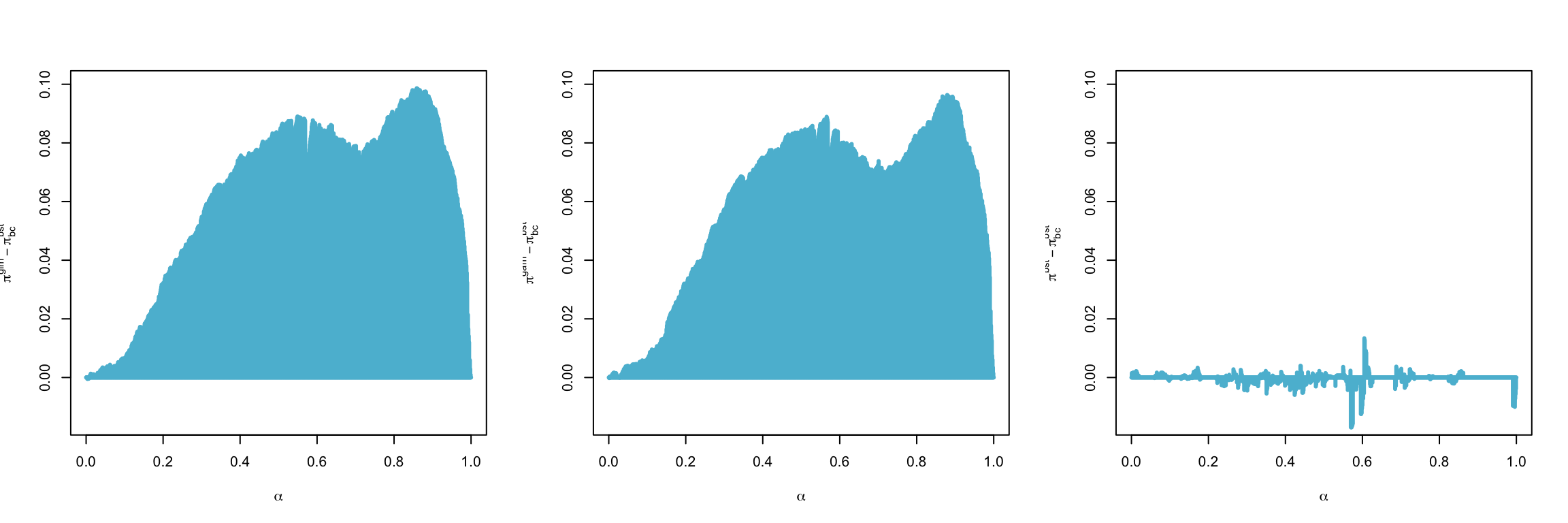}
    \caption{Concentration curves, with $\widehat{\pi}^{\text{\sffamily glm}}$, $\widehat{\pi}^{\text{\sffamily gam}}$, $\widehat{\pi}^{\text{\sffamily bst}}$, against $\widehat{\pi}_{\text{BC}}^{\text{\sffamily bst}}$, from the left to the right on top and $\widehat{\pi}^{\text{\sffamily glm}}-\widehat{\pi}_{\text{BC}}^{\text{\sffamily bst}}$, $\widehat{\pi}^{\text{\sffamily gam}}-\widehat{\pi}_{\text{BC}}^{\text{\sffamily bst}}$ and $\widehat{\pi}^{\text{\sffamily bst}}-\widehat{\pi}_{\text{BC}}^{\text{\sffamily bst}}$ from the left to the right, on the bottom.}
    \label{fig:CASDataset:CC:1}
\end{figure}

\subsection{Claims totals}

Let us now consider aggregate losses and Tweedie deviance with $\xi\in(1,2)$. Three models are considered: a generalized linear model $\widehat{\pi}^{\text{\sffamily glm}}$, a generalized additive model $\widehat{\pi}^{\text{\sffamily gam}}$ where continuous features are transformed nonlinearly using spline functions, and a boosting algorithm $\widehat{\pi}^{\text{\sffamily bst}}$. For the boosting models, we used the \texttt{TDboost} package. We did remove all observations with an exposure not close to 1, to avoid major corrections when the exposure was too small (and additional uncertainty).

On Figure \ref{fig:CASDataset:hist:1:tw}, we can visualise the distribution of predictions $\widehat{\pi}(\boldsymbol{x}_i)$, with the three models.
A grid search lead to $\widehat{\xi}=1.6$.
On Figure \ref{fig:CASDataset:2:tw}, we can visualize the evolution of $s\mapsto \Esp[Y|\widehat{\pi}(\boldsymbol{X})=s]$ when $s$ varies between 0 and 300. Observe that 300 is about 4 times the average value on the training dataset ($\overline{y}=84.9$).

\begin{figure}[!h]
    \centering
    \includegraphics[width=\textwidth]{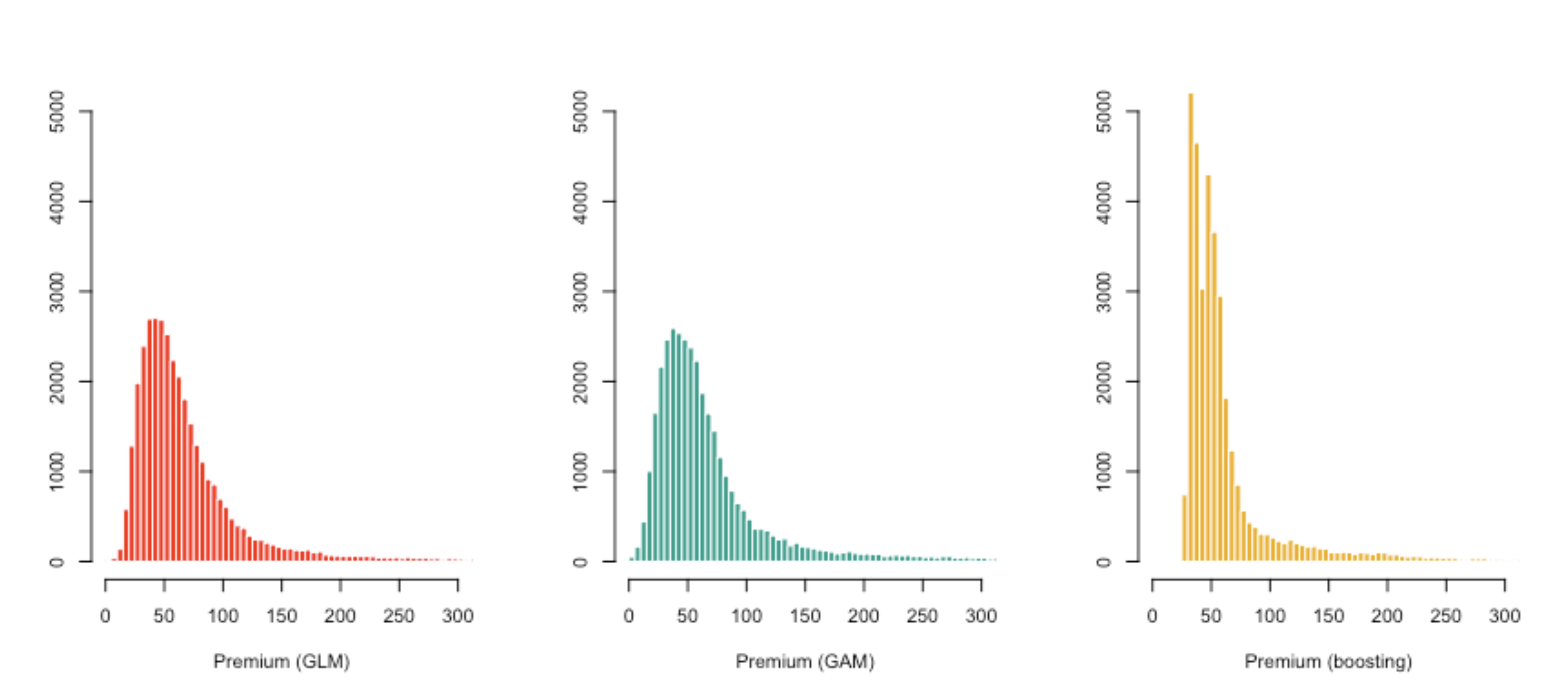}
    \caption{Histogram of $\{\widehat{\pi}(\boldsymbol{x}_1),\cdots,\widehat{\pi}(\boldsymbol{x}_n)\}$ on the validation dataset. In this section, \includegraphics[width=.2cm]{figures/couleur-rouge.png} corresponds to the standard Poisson regression (GLM), \includegraphics[width=.2cm]{figures/couleur-vert.png} corresponds to the smooth additive regression (GAM), and \includegraphics[width=.2cm]{figures/couleur-jaune.png} is the boosting model, with a Tweedie loss function.}
    \label{fig:CASDataset:hist:1:tw}
\end{figure}

\begin{figure}[!h]
    \centering
    \includegraphics[width=\textwidth]{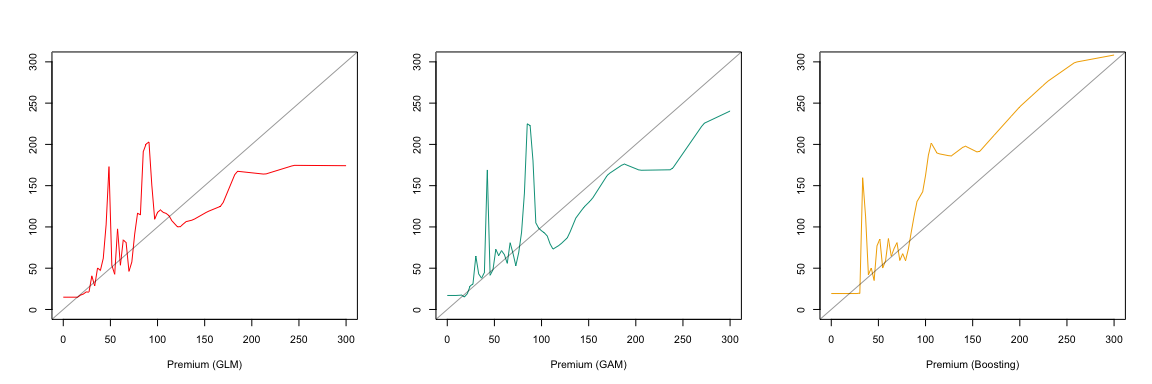}
    \caption{Evolution of $s\mapsto \Esp[Y|\widehat{\pi}(\boldsymbol{X})=s]$ (the thin straight line corresponds to $\mu=\widehat{\pi}$). }
    \label{fig:CASDataset:2:tw}
\end{figure}

On Figure \ref{fig:CASDataset:3:tw}, we illustrate the impact of an underestimation of the power index, with respectively 1.5 and 1.4 (instead of 1.6) for the GLM and GAM procedures. We can see there that the results remain stable when the value of $\xi$ changes. On Figure \ref{fig:CASDataset:4:tw} we exclude policies with an annual loss exceeding 10,000. Note that those represent 0.0539\% of the training dataset (10,000 is the 99\% quantile of losses, for policies experiencing a loss).

\begin{figure}[!h]
    \centering
    \includegraphics[width=\textwidth]{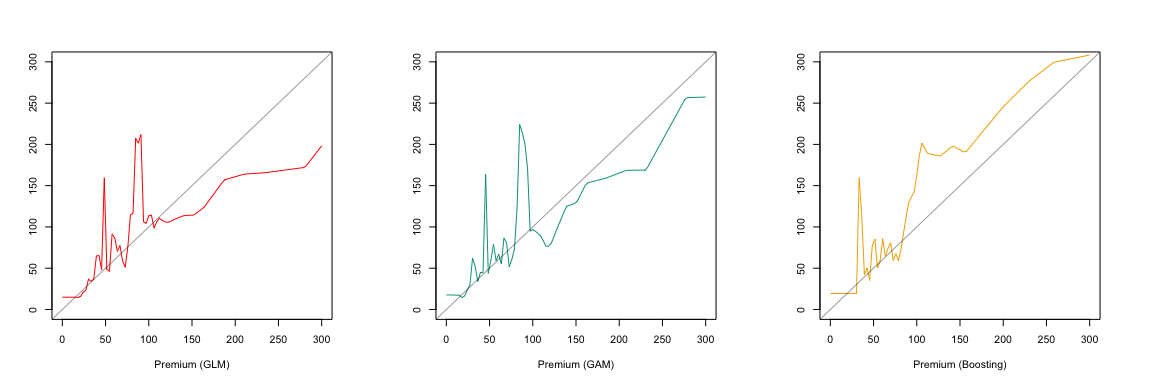}
    \includegraphics[width=\textwidth]{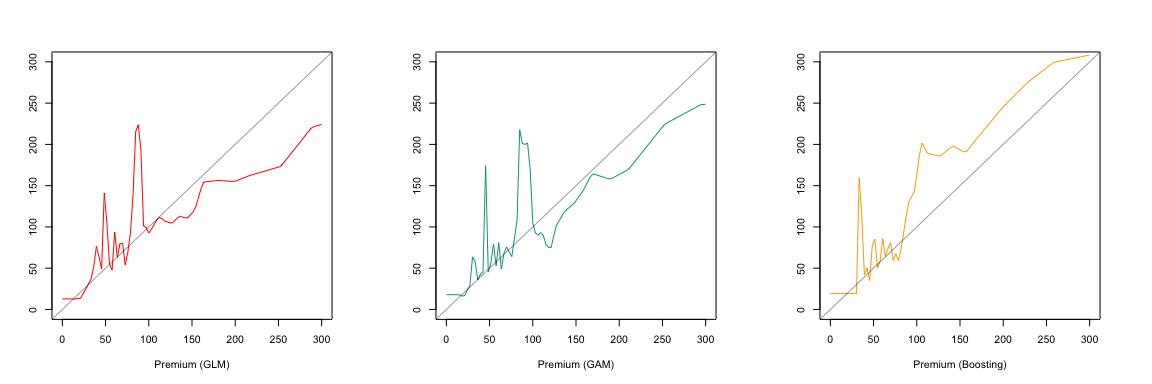}
    \caption{Evolution of $s\mapsto \Esp[Y|\widehat{\pi}(\boldsymbol{X})=s]$, when the Tweedie power is 1.5 on top, and 1.4 below. }
    \label{fig:CASDataset:3:tw}
\end{figure}


\begin{figure}[!h]
    \centering
    \includegraphics[width=\textwidth]{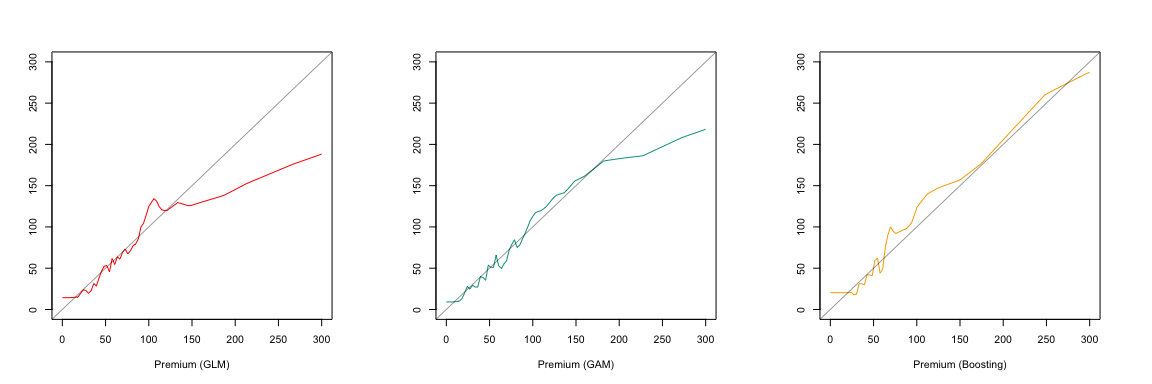}
    \caption{Evolution of $s\mapsto \Esp[Y|\widehat{\pi}(\boldsymbol{X})=s]$, when online policies with claims below 10,000 are kept in the training dataset.}
    \label{fig:CASDataset:4:tw}
\end{figure}

\section{Discussion}

The main message of this paper is as follows. 
Advanced learning models are able to produce scores that better correlate
with the response, as well as with the true premium compared to classical GLMs. 
This comes from the additional freedom obtained by letting scores to
depend in a flexible way of available features, not only linearly. But breaking the overall balance is
the price to pay for this higher correlation. 
Because no constraint on the replication of the observed total, or global balance is imposed, 
machine learning tools are also able to substantially increase overall bias.

To prevent this to occur,
the balance-corrected version of any predictor can be obtained by local GLM, recognizing the
nature of the response $Y$: local Poisson GLM for claim counts,
local Gamma GLM for average claim severities and compound Poisson sums with Gamma-distributed
terms for claim totals.  
The canonical link function is adopted so that maximum-likelihood estimates replicate observed totals.
An intercept-only GLM is fitted locally, with rectangular weight function, on a reduced set of observations
consisting in observed responses and candidate premiums,
with proximity assessed with the help of the candidate premium to be corrected.

In this approach, the supervised learning model is used to produce a real-valued 
signal $\widehat{\pi}(\Xvec)$, reducing the high-dimensional
feature space to the real line. In fact, the score of the model is enough to 
compute $\widehat{\pi}_{\text{BC}}$ so that
the learning model is essentially used to assess proximity among individuals, before performing local averaging.
In that respect, the proposed approach shares some similarities with $k$-NN, or $k$ nearest-neighbors except
that here, proximity is assessed with the help of the real-valued $\widehat{\pi}$ produced by the learning model.
And the proposed extra autocalibration step also counteracts possible overfitting by locally
averaging the initial predictor.

The approach proposed in this paper reconciles minimum bias and method of marginal totals, at the origin of insurance
risk classification, with modern learning tools. With minimum bias, the amount of premium is computed in order to compensate 
insurance risks within meaningful sub-portfolios, maintaining an overall balance.
The proposed balance correction mechanism implements the very same idea, using machine learning
tools to define meaningful neighborhoods where collected premiums must match claims to be compensated. 
This is in accordance with the fundamental mutuality principle at the heart of insurance.
Autocalibration applies very generally, to any machine learning model and only requires a statistical smoother like locfit.

\section*{Acknowledgements}

The authors thank two anonymous Referees and the Editor for their constructive comments which greatly helped to improve this paper.

\section*{References}

\begin{enumerate}
\item[-]
Bailey, R.A. (1963). 
Insurance rates with minimum bias. 
Proceedings of the Casualty Actuarial Society 50, 4-11.

\item[-]
Bailey, R.A., Simon, L.J. (1960). 
Two studies on automobile insurance ratemaking. 
ASTIN Bulletin 1, 192-217.
\item[-] Charpentier, A. (2014). Computational Actuarial Science with R. Chapman and Hall / CRC.  

\item[-]
Delong, L., Lindholm, M., W\"uthrich, M. V. (2021). 
Making Tweedie's compound Poisson model more accessible. 
European Actuarial Journal, in press.

\item[-]
Denuit, M., Dhaene, J., Goovaerts, M.J., Kaas, R. (2005).
Actuarial Theory for Dependent Risks: Measures, Orders and Models. 
Wiley, New York.


\item[-]
Denuit, M., Sznajder, D., Trufin, J. (2019). 
Model selection based on Lorenz and concentration curves, Gini indices and convex order. 
Insurance: Mathematics and Economics 89, 128-139.





\item[-]
Kruger, F., Ziegel, J.F. (2020).
Generic conditions for forecast dominance.
Journal of Business \& Economic Statistics, in press.

\item[-]
Loader, C. (1999). 
Local Regression and Likelihood.
Springer, New York.


\item[-]
Mildenhall, S.J. (1999). 
A systematic relationship between minimum bias and generalized linear models. 
Proceedings of the Casualty Actuarial Society 86, 393-487.



\item[-]
Savage, L.J. (1971). 
Elicitation of personal probabilities and expectations. 
Journal of the American Statistical Association 66, 783-810.

\item[-]
Shaked, M., Shanthikumar, J.G. (2007).
Stochastic Orders.
Springer, New York.

\item[-]
Shaked, M., Sordo, M. A., Suarez-Llorens, A. (2012). 
Global dependence stochastic orders. 
Methodology and Computing in Applied Probability 14, 617-648.

\item[-]
Schelldorfer, J., W\"uthrich, M.V. (2019).
Nesting classical actuarial models into Neural Networks. Available at SSRN: https://ssrn.com/abstract=3320525

\item[-]
Wright, R. (1987).
Expectation dependence of random variables, with an application in portfolio theory.
Theory and Decision 22, 111-124.

\item[-]
W\"uthrich, M.V. (2019). 
From Generalized Linear Models to Neural Networks, and back. 
Available at SSRN: https://ssrn.com/abstract=3491790 
\item[-]
W\"uthrich, M.V. (2020). 
Bias regularization in neural network models for general insurance pricing. 
European Actuarial Journal 10, 179-202.
\item[-]
W\"uthrich, M.V. (2021).
The balance property in neural network modelling.
Statistical Theory and Related Fields, in press.  

\item[-]
Zumel, N. (2019).
An ad-hoc method for calibrating uncalibrated models.
Win-Vector Blog, WordPress.

\end{enumerate}

\end{document}